\documentclass{article}

\usepackage{arxiv}

\usepackage[utf8]{inputenc} % allow utf-8 input
\usepackage[T1]{fontenc}    % use 8-bit T1 fonts
\usepackage{hyperref}       % hyperlinks
\usepackage{url}            % simple URL typesetting
\usepackage{booktabs}       % professional-quality tables
\usepackage{amsfonts}       % blackboard math symbols
\usepackage{nicefrac}       % compact symbols for 1/2, etc.
\usepackage{microtype}      % microtypography
\usepackage{lipsum}

\usepackage{cite}
\usepackage{amsmath,amssymb,amsfonts}
\usepackage{graphicx}
\usepackage{algorithm}
\usepackage[noend]{algpseudocode}
\usepackage{hyperref}
\usepackage{textcomp}
\usepackage{threeparttablex} % for table footnotes

\usepackage{mathtools}
\usepackage{caption}
\usepackage{float}
\usepackage{stmaryrd}
\usepackage{epstopdf}
\usepackage{multirow}
\usepackage{pifont}
\usepackage{caption}
\usepackage{subcaption}
\usepackage{xcolor}

\newcommand{\norm}[1]{\left\lVert#1\right\rVert}

\newcommand{\R}{{\mathbb{R}}}

\newcommand{\B}{{\mathcal B}}

\newcommand{\N}{{\mathbb{N}}}

\newcommand{\e}{\mathsf{e}}

\newcommand{\T}{{\mathbf{T}}}
\newcommand{\So}{{\mathbf{S}}}
\newcommand{\Obs}{{\mathcal{U}}}

\newcommand{\U}{{\mathbf{U}}}
\newcommand{\Prob}{{\mathbb{P}}}

\newcommand{\cen}{{\mathsf{\textbf{c}}}}
\newcommand{\rad}{{\mathsf{\textbf{r}}}}

\newtheorem{theorem}{Theorem}[section]

\newtheorem{assumption}{Assumption}

\newtheorem{proposition}[theorem]{Proposition}
\newtheorem{definition}[theorem]{Definition}
\newtheorem{lemma}[theorem]{Lemma}
\newtheorem{remark}[theorem]{Remark}
\newtheorem{problem}[theorem]{Problem}
\newenvironment{proof}{\par\noindent\textbf{Proof.} }{\hfill$\blacksquare$\par}

\title{Spatiotemporal Tubes for Probabilistic Temporal Reach-Avoid-Stay Task in Uncertain Dynamic Environment
\thanks{ This work was supported in part by the SERB Start-Up Research Grant; in part by the ARTPARK. The work of Ratnangshu Das was supported by the Prime Minister’s Research Fellowship from the Ministry of Education, Government of India.}
}

\author{
 Siddhartha Upadhyay \\
  Robert Bosch Centre for Cyber-Physical Systems\\
  IISc, Bengaluru, India\\
  \texttt{siddharthau@iisc.ac.in} \\
   \And
 Ratnangshu Das \\
  Robert Bosch Centre for Cyber-Physical Systems\\
  IISc, Bengaluru, India\\
  \texttt{ratnangshud@iisc.ac.in} \\
   \And
 Pushpak Jagtap \\
  Robert Bosch Centre for Cyber-Physical Systems\\
  IISc, Bengaluru, India\\
  \texttt{pushpak@iisc.ac.in} \\
}

\begin{document}
\maketitle

\begin{abstract}
In this work, we extend the Spatiotemporal Tube (STT) framework to address Probabilistic Temporal Reach–Avoid–Stay (PrT-RAS) tasks in dynamic environments with uncertain obstacles. We develop a real-time tube synthesis procedure that explicitly accounts for time-varying uncertain obstacles and provides formal probabilistic safety guarantees. The STT is formulated as a time-varying ball in the state space whose center and radius evolve online based on uncertain sensory information. We derive a closed-form, approximation-free control law that confines the system trajectory within the tube, ensuring both probabilistic safety and task satisfaction. Our method offers a formal guarantee for probabilistic avoidance and finite-time task completion. The resulting controller is model-free, approximation-free, and optimization-free, enabling efficient real-time execution while guaranteeing convergence to the target. The effectiveness and scalability of the framework are demonstrated through simulation studies and hardware experiments on mobile robots, a UAV, and a 7-DOF manipulator navigating in cluttered and uncertain environments.
\end{abstract}


%%%%%%%%%%%%%%%%%%%%%%%%%%%%%%%%%%%%%%%%%%%%%%%%%%%%%%%%%%%%%%%%%%%%%%%%%%%%%%%%
%%--------------------------------NEW SECTION---------------------------------%%
%%%%%%%%%%%%%%%%%%%%%%%%%%%%%%%%%%%%%%%%%%%%%%%%%%%%%%%%%%%%%%%%%%%%%%%%%%%%%%%%
\section{Introduction}
\label{sec1}
Safety is a key requirement in the development of autonomous systems, especially in safety-critical applications such as autonomous driving, surgical robotics, unmanned aerial vehicles, and industrial automation. These applications demand the systems to offer reliable, precise, and consistent performance in hazardous environments \cite{safety_application}. To achieve these capabilities, autonomous platforms typically rely on sensors such as LiDAR, cameras, GPS, and radar to perceive the surrounding environment \cite{sensor}, allowing reliable obstacle detection, safe decision-making and control \cite{autonomous_1}.

However, autonomous systems still face two major sources of uncertainty: model uncertainty and perception uncertainty. Obtaining an exact model of a robotic platform is often difficult in practice. Unmodeled dynamics, parameter variations, actuator nonlinearities, and external disturbances can all significantly degrade model accuracy \cite{robot_model_1}. {The second major challenge comes from uncertainty in how the system perceives its environment \cite{robotic_perception}. Sensor performance varies with modality and conditions. LiDAR can produce weak or sparse returns in adverse weather or with reflective surfaces \cite{lidar}, cameras are sensitive to lighting and motion \cite{camera}, and RGB-D or infrared sensors often fail in transparent or outdoor scenes \cite{rgbd}. These limitations highlight the need for control frameworks that ensure safety despite uncertainty in both system dynamics and perception. Although many methods address one of these uncertainties, only a few consider both simultaneously, particularly in safety-critical settings.}
% The second major issue arises from uncertainty in how the system perceives its environment \cite{robotic_perception}. This perception uncertainty is inherent in autonomous systems and arises from factors that vary depending on the sensing modality and environmental conditions. For example, LiDAR performance can degrade in adverse weather, over long distances, or when interacting with dark or reflective surfaces, leading to weak or sparse point-cloud returns \cite{lidar}. Camera-based systems are highly sensitive to lighting conditions and suffer from noise, overexposure, or blur in low light, bright sunlight, or during fast motion \cite{camera}. Similarly, RGB-D and infrared-based sensors often fail in transparent, glossy, or outdoor scenes due to interference and poor depth recovery \cite{rgbd}. These limitations make it essential to develop control frameworks that can guarantee safety despite uncertainty in both the system model and the system's perception of the environment. While several existing approaches address uncertainty in either the system dynamics or the perception module, only a few works consider both sources of uncertainty simultaneously, especially in safety-critical applications.

To synthesize safe control for safety-critical systems, Control Barrier Functions (CBFs) \cite{cbf} have emerged as one of the most promising frameworks in the literature. However, the standard formulation of CBFs is not inherently robust to model uncertainty or to uncertainty arising from perception errors. Several extensions of CBF have been developed to improve robustness under uncertain dynamics \cite{robust_2, robust_cbf}, but do not handle uncertainty in perception. In \cite{sensor_dist}, a distributionally robust optimization (DRO) CBF-framework is proposed that uses noisy sensor and state-estimation samples to ensure probabilistic safety in dynamic environments, but suffers from sample complexity and does not scale well for higher dimensional systems.

An alternative way to handle environmental uncertainty in motion planning is to impose chance constraints that bound the probability of collision. In \cite{ccp}, the authors derive closed-form collision probabilities that capture uncertainty in both robot and obstacle states, enabling efficient probabilistic collision checking in dynamic environments; related ideas also appear in CC-RRT \cite{ccp_rrt}. Chance constraints have further been incorporated into trajectory optimization frameworks, such as \cite{chance_mpc}, which extends the Unscented Model Predictive Path Integral method \cite{umppi} to $C^2$U-MPPI by embedding probabilistic collision checks. Similar formulations have been used in human–robot collaboration, where collision probability is encoded as an uncertain CBF constraint \cite{ucbf}. While these methods explicitly integrate uncertainty into the optimization, they often incur high computational costs due to the sample-based evaluation of constraints and typically require accurate system dynamics.

The Spatiotemporal Tube (STT) framework is one of the most effective approaches for handling system uncertainty without requiring explicit dynamics. STT \cite{das2024spatiotemporal} offers a robust, model-free, approximation-free formulation for fully actuated systems, providing closed-form controllers that address parametric and disturbance uncertainties. Recent work in \cite{STT_real_time} extends STT to real-time Temporal Reach-Avoid Stay (T-RAS) tasks \cite{jagtap2020formal} via online tube synthesis. Unlike A* or RRT planners that need separate controllers, potential fields without guarantees, or MPC or CBF methods that require accurate models, STT stays robust without relying on a system model. However, the real time STT framework in \cite{STT_real_time} assumes perfect perception and does not incorporate uncertainty in obstacle locations, limiting its applicability when obstacle positions are imprecisely known. Furthermore, its prescribed-time tube synthesis may yield non-smooth or infeasible trajectories in real-time settings where obstacle information evolves online.

% One of the most efficient ways to handle system uncertainty and design controllers without relying on explicit system dynamics is the Spatiotemporal Tube (STT) framework. STT provides a robust and principled method for addressing parametric and disturbance-related uncertainties \cite{STT, das2024spatiotemporal}, offering model-free, approximation-free, closed-form controllers for fully actuated systems. Recent work in \cite{stt} extended real-time implementation of STT for Temporal-Reach Avoid task (T-RAs) specification \cite{jagtap2020formal} through online tube synthesis. Unlike path-planning approaches such as A*, RRT, or RRT*, which require an external controller, and unlike APF, which lacks formal guarantees, or model-based frameworks like MPC and CBFs that rely on accurate system models, the STT approach does not require an explicit model and remains robust to bounded disturbances. However, the method in \cite{} provides guarantees only under perfect perception and does not account for measurement uncertainty in obstacle locations, making it unsuitable when obstacle positions are not known exactly. Moreover, its prescribed-time controller synthesis can lead to non smooth or infeasible trajectories in real-time settings where obstacle information is revealed only online.

In this work, we aim to bridge the gap in jointly managing environmental uncertainty and unknown system dynamics within a single unified framework. We assume that the robot’s state can be measured accurately using high-precision localization sensors such as GPS. In contrast, the robot doesn't have access to the true position of the obstacle in the environment, but instead has access to a probability distribution, such as a distribution that can be produced by SLAM algorithms \cite{aulinas2008slam} and typically consists of a Gaussian distribution over the obstacles in the environment. 
 We extend the Spatiotemporal Tube (STT) methodology to solve temporal reach–avoid–stay (T-RAS) tasks under \textit{uncertain environments}, while providing \textit{probabilistic guarantees} for avoiding uncertain unsafe sets.
Our main contributions are summarized as follows:
\begin{enumerate}
  \item We extend the real-time STT based framework to solve the \textit{Temporal Reach–Avoid–Stay (T-RAS)} in the presence of a dynamically uncertain unsafe set and reformulate the problem as a \textit{Probabilistic Temporal Reach–Avoid–Stay (PrT-RAS)} task, where each dynamic unsafe set is associated with a prior known time varying uncertainty level and a user-defined minimum probability of avoidance.
 \item {We propose a framework that incorporates dynamically uncertain unsafe sets for real-time tube synthesis, ensuring PrT-RAS satisfaction with probabilistic safety guarantees, and we relax the prescribed-time requirement of \cite{STT_real_time} by generating tubes that ensure finite-time convergence, making the approach more suitable for dynamic environments.}
  \item We derive an \textit{approximation-free}, \textit{model-free}, and \textit{optimization-free} \textit{closed-form controller} that constrains the system trajectory within a real-time evolving tube, ensuring that the agent avoids dynamic, uncertain unsafe sets with a user-specified minimum probability.
  % Although DRO-based methods offer robustness to distributional uncertainty, they typically require solving optimization problems that scale with the number of uncertainty samples, leading to significant computational overhead and the control framework involving DRO requires system model.
  \item We validate our framework through extensive simulations and hardware experiments, including a 2D omnidirectional robot navigating dynamic uncertain obstacles, a 3D UAV case study, and hardware tests on a 2D robot and a 7-DOF manipulator, also demonstrating robustness to external disturbances.
\end{enumerate}

\section{Preliminaries and Problem Formulation}
\label{sec:prelim}
\subsection{Notation}
% The symbols $\N$, $\N_0$, $ \R$, $\R^+$, and $\R_0^+ $ denote the set of natural, whole, real, positive real, and nonnegative real numbers, respectively. 
% For $a,b\in\R$ and $a< b$, we use $(a,b)$ to represent open interval in $\R$. 
For $a,b\in\N$ and $a\leq b$, we use $[a;b]$ to denote close interval in $\N$.
% A vector space of real matrices with $ n $ rows and $ m $ columns is denoted by $ \R^{n\times m} $. A space of column vectors with $n$ rows is represented by $ \R^{n}$.
% A vector $x \in \mathbb{R}^{n}$ with entries $x_1, \ldots, x_n$ is represented as $[x_1, \ldots, x_n]^\top$, where $x_i \in \mathbb{R}$ denotes the $i$-th element of the vector $x\in\mathbb{R}^n$ and $i \in [1;n]$. We represent the Euclidean norm using $\norm{\cdot}$.
% Given a matrix $M\in\R^{n\times m}$, $M^\top$ represents the transpose of the matrix $M$. 
% The power set of a set \textbf{A} is defined as $\mathcal{P}$(\textbf{A}).
% Given $N \in \N$ sets $\textbf{A}_i$, $i\in\left[1;N\right]$, we denote the Cartesian product of the sets by $\textbf{A}=\prod_{i\in\left[1;N\right]}\textbf{A}_i:=\{(a_1,\ldots,a_N)|a_i\in \textbf{A}_i,i\in\left[1;N\right]\}$. 
% The space of bounded continuous functions is denoted by $\mathcal{C}$.
A ball centered at $\cen \in \mathbb{R}^n$ with radius $\rad \in \mathbb{R}^+$ is defined as $\mathcal{B}(\cen, \rad) := \{ x \in \mathbb{R}^n \mid \|x - \cen\| \leq \rad \}$. We define $\Prob$ as a probability measure induced by a continuous random variable $Z$ for a measurable set $A\subset\R^n$ and is represented as $\Prob(Z \in A)$. 
We use $<a,b>$ to denote the inner product. We use $x\circ y$ to represent the element-wise multiplication where $x,y\in \R^n$. $I_n$ is identity matrix of order $n\in \N$. We use $\mathcal{N}(x,\Sigma)$ to represented a normal distribution of $n\in \N$ dimensional gaussian vector with mean $\mu\in \R^n$ and a symmetric positive definite covariance matrix $\Sigma \in R^{n\times n}$.
All other notations in this paper follow standard mathematical conventions.

\subsection{System Definition}

We consider a class of uncertain control-affine, multi-input multi-output (MIMO), nonlinear pure-feedback systems:
\begin{align} \label{eqn:sysdyn}
  %\mathcal{S}: \dot{x}^{(1)} &= f_1(x_1) + g_1(x_1)x_2 + w_1 \notag\\
  %\dot{x}^{(2)} &= f_2(x_1,x_2) + -g_2(x_1,x_2)x^{(3)} + w_2 ... \notag\\
  &\dot{x}_i(t) = f_i(\overline{x}_i(t)) + g_i(\overline{x}_i(t))x_{i+1}(t) + w_i(t), i\in [1;N-1], \notag\\
  &\dot{x}_{N}(t) = f_N(\overline{x}_N(t)) + g_N(\overline{x}_N(t))u(t) + w_N(t), \\
  &y(t) = x_1(t), \nonumber
\end{align}
where for each $t\in\R^+_0$ and $i\in[1;N]$,
\begin{itemize}
  \item $x_i(t) = [x_{i,1}(t), \ldots, x_{i,n}(t)]^\top \in\mathbb{R}^{n}$ is the state,
  \item $\overline{x}_i(t) := [x_1^\top(t),...,x_i^\top(t)]^\top \subset \mathbb{R}^{ni} $,
  \item $u(t) \in \mathbb{R}^n$ is the control input vector,
  \item $w_i(t) \in \mathbf{W} \subset \R^n$ is the unknown bounded disturbance, and
  \item $y(t) = [x_{1,1}(t), \ldots, x_{1,n}(t)]\in \R^n$ is the output.
\end{itemize}

The functions $f_i: \R^{ni} \rightarrow \mathbb{R}^n$ and $g_i: \R^{ni} \rightarrow \mathbb{R}^{n \times n}$ satisfy the following assumptions:
\begin{assumption}\label{assum:lip}
  For all $i \in [1;N]$, the functions $f_i$ and $g_i$ are unknown but locally Lipschitz continuous.
\end{assumption}
\begin{assumption}[\cite{PPC1,PPC0}] \label{assum:pd}
  For all $\overline{x}_i \in \R^{ni} $, the symmetric part of $g_i$, defined as $g_{i,s}(\overline{x}_i) := \frac{g_i(\overline{x}_i)+g_i(\overline{x}_i)^\top}{2}$ is uniformly sign definite with known sign. Without loss of generality, we assume $g_{i,s}(\overline{x}_i)$ is positive definite, that is, there exists a constant $\underline{g_i}\in\mathbb R^+, \forall i \in [1;N]$ such that
  $0 < \underline{g_i} \leq \lambda_{\min} (g_{i,s}(\overline{x}_i)), \forall \ \overline{x}_i \in \R^{ni},$
  where $\lambda_{\min}(\cdot)$ denotes the smallest eigenvalue of a matrix.
\end{assumption}
This assumption ensures that in \eqref{eqn:sysdyn} global controllability is guaranteed, i.e., $g_{i,s}(\overline{x}_i) \neq 0,$ for all $\overline{x}_i \in \R^{ni} $.

\subsection{Problem Formulation}
In this work, we consider that the output of the system \eqref{eqn:sysdyn}, $y(t)$, is subject to a \textit{temporal reach-avoid-stay (T-RAS) specification} under an uncertain environment. To formally define a probabilistic task, we first define the time-varying uncertain unsafe set below.
\begin{definition}[Time-varying Uncertain Unsafe Set]\label{defn:obs}
We define the

\hspace{-0.8cm} uncertain unsafe set $\U(t)$ as the union of $n_o$ obstacles, each represented as a time-varying closed ball: 
  \begin{align}
      \U(t) = \bigcup_{j=1}^{n_o} \Obs^{(j)}(t), \quad \text{where} \quad \Obs^{(j)}(t):= \mathcal{B}(O_p^{(j)}(t), \rad_o^{(j)}(t)).\notag
  \end{align}
  The center of the $j$-th obstacle $\Obs^{(j)}(t)$ is a random vector 
  $$O_p^{(j)}(t) \sim \mathcal{N}\Big(\mu^{(j)}(t), \Sigma^{(j)}(t)\Big),$$ 
  where $\mu^{(j)}(t)\in \mathbb{R}^n$ is the mean and $\Sigma^{(j)}(t):=(\sigma^{(j)}(t))^2\mathbb{I}_{n\times n}$ with $(\sigma^{(j)}(t))^2 \in \R^+$ is an isotropic covariance matrix, chosen to model uncertainty that is direction-independent around each obstacle center and $\rad_o^{(j)}(t)$ is the radius of the obstacle we assume that there is no uncertainty in the radius of the obstacle. Since these regions are defined independently, they allow modeling multiple disconnected and dynamically evolving uncertain unsafe regions.
\end{definition}
% \begin{assumption}\label{obs_unc}
%   We define the unsafe set $\U(t)$ as the union of $n_o$ obstalces, each represented as a time-varying closed ball: 
%   $$\U(t) = \bigcup_{j=1}^{n_o} \Obs^{(j)}(t), \quad \text{where} \quad \Obs^{(j)}(t):= \mathcal{B}(O_p^{(j)}(t), \rad_o^{(j)}(t)).$$ 
%   The center of the $j$-th obstacle $\Obs^{(j)}(t)$ is a random vector 
%   $$O_p^{(j)}(t) \sim \mathcal{N}\Big(\mu^{(j)}(t), \Sigma^{(j)}(t)\Big),$$ 
%   where $\mu^{(j)}(t)\in \mathbb{R}^n$ is the mean and $\Sigma^{(j)}(t):=(\sigma^{(j)}(t))^2\mathbb{I}_{n\times n}$ with $(\sigma^{(j)}(t))^2 \in \R^+$ is an isotropic covariance matrix, chosen to model uncertainty that is direction-independent around each obstacle center. Since these regions are defined independently, it allows modeling multiple disconnected and dynamically evolving uncertain unsafe regions.
% \end{assumption}
\begin{remark}
 { The obstacle positions are treated as uncertain because real-world perception systems provide only approximate estimates of the environment. Noise in sensors such as LiDAR and RGB-D cameras, limited resolution, and measurement bias, all contribute to uncertainty in the perceived obstacle locations, making a probabilistic representation necessary \cite{uncertainity}.}
\end{remark}
\begin{definition}[Probabilistic Temporal Reach-Avoid-Stay]\label{def:prtras}
Given the output-space $\R^n$, a time-varying uncertain unsafe set $\U: \R_0^+ \rightarrow \R^n$ as
defined in Definition \eqref{defn:obs}, an initial set $\So \subset \R^n$ with $\Prob\Big(\norm {s-O_p^{(j)}(t)}\geq\rad^{(j)}_o(t)\Big)\geq\varepsilon^{(j)} ,\forall s\in \mathbf{S}, \varepsilon^{(j)}\in(0,1],j \in [1;n_o]$ and a target set $\T \subset \R^n $, we say the output $y(t)$ satisfies the \textit{PrT-RAS} task if: 
\begin{align}
  & y(0) \in \So,\notag\\
  & \exists t_c\in \R_0^+ \text{ s.t. } y(t) \in \T,\forall t \in [t_c,\infty),\notag\\
  &\Prob\Big(\norm{y(t)-O_p^{(j)}(t)}\geq\rad^{(j)}_o(t)\Big)\geq\varepsilon^{(j)},\forall j \in [1;n_o], \forall t \in \R_0^+.
\end{align}
where $\varepsilon^{(j)}$ is the user defined value that specifies the minimum probability with which $j^{th}$ obstacle should be avoided. 
% $$y(0) \in \So, \quad y(t) \in \T,\forall t \in [t',\infty), \quad and \quad \Prob(O_p^{(j)}\notin \B (y(t),\rad^{(j)}_o))>\varepsilon^{(j)},\forall \Obs^{(j)}\in\U, \forall t \in \R_0^+.$$
\end{definition}

We will now state the main control problem addressed in this work.
\begin{problem}[Real-time Control under Uncertainty]\label{prob1}
Given the uncertain system dynamics in \eqref{eqn:sysdyn} under Assumptions \ref{assum:lip} and \ref{assum:pd}, and a PrT-RAS task as defined in Definition \ref{def:prtras}, synthesize a \textit{real-time}, \textit{approximation-free}, and \textit{closed-form} control law $u(t)$ such that the resulting output trajectory $y(t)$ satisfies the PrT-RAS specification.
\end{problem}

To solve the above problem, we utilize the STT-based framework, which prescribes a time-varying region in the output space that remains probabilistically safe and goal-directed throughout the horizon.
\begin{definition}\label{STT-defination}
  Given a \textbf{PrT-RAS} in Definition \ref{def:prtras}, an STT is defined as a time-varying ball
  $\Gamma(t)=\B(\cen(t),\rad(t)),$ 
  with center $\cen:\mathbb{R}_0^+ \rightarrow\mathbb{R}^n$ and radius $\rad:\mathbb{R}_0^+ \rightarrow\mathbb{R}^+$, such that:
  \begin{subequations}
    \begin{align}
      &\rad(t)
\in\mathbb{R}^+,\quad \forall t \in\R_0^+,\\
 &\Gamma(0) \subset \mathbf{S},\\
 &\exists t_c\in \R_0^+ \text{ s.t. }\Gamma(t_c) \subset \mathbf{T},\forall t\in [t_c,\infty),\\
  % &\Prob(\Gamma(t) \cap \Obs^{(j)}=\emptyset)\geq\varepsilon^{(j)},\forall\Obs^{(j)}\in \U,t\in\R_0^+\\
  &\Prob\Big(\norm{\cen(t)-O_p^{(j)}(t)}\geq\rad^{(j)}_o(t)+\rad(t)\Big)\geq\varepsilon^{(j)},\forall j \in [1;n_o], \forall t \in \R_0^+.
 \end{align}
  \end{subequations}
\end{definition}
% \begin{remark}
% Instead of enforcing convergence to the target set at a fixed prescribed time (e.g., \cite{STT_real_time}), Definition~\ref{def:prtras} only requires reaching the target set in some \emph{finite} time.
% This relaxation is particularly appropriate in dynamic environments where obstacle information evolves and is not known a priori. Enforcing convergence at a fixed time may demand unrealistically fast motion of the tube center $\cen(t)$, potentially causing non-smooth or infeasible trajectories. Finite-time convergence removes these issues and provides a more realistic requirement.
% \end{remark}
\begin{remark}
{Unlike prescribed-time convergence used in \cite{STT_real_time}, Definition~\ref{def:prtras} only requires the system to reach the target set in finite time. This is more suitable for dynamic environments where obstacle information changes over time; enforcing a fixed convergence time can force the tube center $\cen(t)$ to move unrealistically fast, leading to non-smooth or infeasible trajectories.}
\end{remark}

\section{Designing Spatiotemporal Tubes} \label{sec:tube}
In this section, the main goal is to construct the STT that starts from the initial set and reaches the target set, while maintaining a minimum positive probability of avoiding the uncertain unsafe set.

We begin by selecting the points $s=[s_1,...,s_n]^\top\in \text{int}(\So)$ and $\eta=[\eta_1,...,\eta_n]^\top\in \text{int}(\T)$, located in the interior of the initial set $\So$ and the target set $\T$, respectively. Around the points, we define the balls $\hat\So :=\B(s,d_S)$ and $\hat\T :=\B(\eta,d_T)$ of radii $d_S,d_T\in \mathbb{R}^+$,
chosen such that $\hat \So\subset\So$ and $\hat\T\subset \T$. 
Additionally, to ensure a safe approach to the target, we introduce the following separation assumption on the probabilistic obstacle locations.
\begin{assumption}\label{ass_pmin}
  After some time $t_1\in \R_0^+$, the unsafe set remain separated from the tube center $\cen(t)$ with at least a known probability $p_{d}^{(j)}\in (\varepsilon^{(j)},1]$:
  $$\Prob\Big(\norm{\cen(t)-O_p^{(j)}(t)}\geq\rad_s^{(j)}(t)\Big)\geq p_d^{(j)},\forall j \in [1;n_o], \forall t \in [t_1 \infty), $$ 
  where $\rad_s^{(j)}(t)=\rad^{(j)}_o(t)+\rad_{min}$ is the safety radius with the buffer margin $\rad_{min}$.
\end{assumption} 

Now, to construct such an STT, we define the evolution laws for its center and radius, designed to ensure that the resulting tube satisfies the PrT-RAS specification.

The center $\cen(t)$ evolves according to the following dynamics:
 \begin{equation}\label{cen_dynamic}
    \dot{\cen}=k_1(\eta-\cen(t))^{\frac{1}{3}}+\sum_{j}^{n_o}\Big (k_{2,j}m^{(j)}(t) + k_{3,j}v^{(j)}(t)\Big ) \theta^{(j)}(t) ,\quad \cen(0)=s \in \text{int}(\So),
  \end{equation}
  where $k_{1}\in \R^+$ and $k_{2,j},k_{3,j}\in \R^+$ are arbitrary positive constants.
  
  The first term in equation \eqref{cen_dynamic} is responsible for pulling the center of the tube towards the target in finite time, and the rate of convergence to the target set depends on $k_1\in \R^+$.
  
  The switching function $\theta^{(j)}$ in \eqref{cen_dynamic} activates the probabilistic obstacle avoidance when the following condition is met:
  $$
\theta^{(j)}(t) = 
\begin{cases}
\frac{1}{q^{(j)}(\cen(t),t)}-\frac{1}{p_d^{(j)}}, & \text{if } q^{(j)}(\cen(t),t) \leq p_d^{(j)} \\
0, & \text{otherwise},
\end{cases}
$$
where $p_d^{(j)}\in(\varepsilon^{(j)},1]$ is a user-defined threshold specifying the minimum acceptable probability of avoiding the $j^{th}$ unsafe set. Thus, the avoidance part is activated only when the probability of collision exceeds $1-p_d^{(j)}$. 
This avoidance term in \eqref{cen_dynamic} is mainly defined by two vectors $m^{(j)}(t)$ and $v^{(j)}(t)\in \R^n$
\begin{align}
  m^{(j)}(t)=\frac{\cen(t)-\mu^{(j)}(t)}{q^{(j)}(\cen(t),t)-\varepsilon^{(j)}},
  \quad \forall j \in \{1, \ldots, n_o\},\nonumber
\end{align}
and $v^{(j)}$ lies in the null space of $m^{(j)}$ i.e., ${m^{(j)}}^\top(t) v^{(j)}(t)=0$. The term $q^{(j)}(x,t):\R^n\times\mathbb{R}^+_0 \rightarrow [0,1],\forall j \in [1,n_o]$ is the probability of avoiding the collision and is given by:
 \begin{align}\label{eqn:prob_q}
  &q^{(j)}(x,t)=1-\mathbb{P}\Big(\norm{x-O_p^{(j)}(t)}<\rad_s^{(j)}(t)\Big).
  \end{align}
\begin{proposition}\label{prop:1}
The probability that the center $\cen(t)$ avoids the unsafe set in \eqref{eqn:prob_q} can be rewritten as follows:
\begin{align}\label{eqn:q_F}
  q^{(j)}(\cen(t),t)=1-F_{\hat Z^{(j)}(t)}\Bigg(\Bigg(\frac{\rad_s^{(j)}(t)}{\sigma^{(j)}(t)}\Bigg)^2;n,\lambda^{(j)}(\cen(t),t)\Bigg),
\end{align}
where $\hat{Z}^{(j)}(t)$ follows a non-central chi-square distribution with degrees of freedom equal to the dimension of the state space and non-centrality parameter $\lambda^{(j)}(\cen(t),t)=\frac{\lVert \cen(t)-\mu^{(j)}(t)\rVert^2}{(\sigma^{(j)}(t))^2}$. The function $F_{\hat{Z}^{(j)}(t)}$ is the cumulative distribution function (CDF) with distribution parameter $n\text{ and }\lambda^{(j)}(\cen
(t),t)$, calculated using the method described in \cite{chi_1}.
\end{proposition}
\begin{proof}
  See Appendix \eqref{appendix_1}
\end{proof}

Next, we describe how the radius of the tube changes over time. The radius $\rad(t)$ is dynamically adjusted according to the proximity of the tube to the region where the probability of avoiding the unsafe set is less than or equal to $\varepsilon^{(j)}$. Its evolution is given by:
\begin{align}\label{eq:rad}
\dot \rad(t)&=\frac{e^{-\nu d(t)}\dot{d}(t)}{(e^{-\nu\rad_{max}+e^{-\nu d(t)}})},
\end{align}
where $\nu \in \R^+$ is a tuning constant that controls how smoothly the radius shrinks and grows. The function $d(t)$ is defined as a smooth approximation of the minimum over the $\hat{d}^{(j)}(t),\forall j \in[1,n_o]$.
\begin{align}
  d(t)&=-\frac{1}{\nu}\Big( \sum_{j=1}^{n_o}e^{-v\hat d^{(j)} (t)} {\hat{d}^{(j)}}(t)\Big),\nonumber\\ 
  \hat{d}^{(j)}(t)&= \sigma^{(j)}(t)\sqrt{F_{\hat Z^{(j)}}^{-1}(1-\varepsilon^{(j)})}-\rad^{(j)}_o(t),\quad \label{eqn:m_crv}
\end{align}
where $\rad_{max} \leq \min(d_S, d_T)$ and $\rad_{min} < \rad_{max}$ are the maximum and minimum allowable tube radius, respectively. The term $F^{-1}_{\hat Z^{(j)}}$ represents the inverse CDF of $\hat Z^{(j)}$, computed following \cite{chi_compute}. Intuitively, $\hat{d}^{(j)}(t)$ quantifies how far the tube center $\cen(t)$ is from the boundary at which the probability of colliding with the $j$-th obstacle reaches $1-\varepsilon^{(j)}$. We then compute a smooth minimum over all such distances $\hat{d}^{(j)}(t)$ for $j \in [1,n_o]$. Finally, to get $\rad(t)$, this aggregated value is smoothly minimized with the user-defined default radius $\rad_{max}$ to limit the radius of the tube to $\rad_{max}$.

% \begin{remark}
% The closed form probability density function of $m^{(j)}$ in~\eqref{eqn:m_crv} is difficult to obtain analytically, but it will be continuous and differentiable in time. So for implementation purpose we approximate $\hat d^{(j)}$by sampling from the Gaussian distribution of $O_p^{(j)} $ and computing the empirical \( (1 - \varepsilon^{(j)}) \) quantile.
% \end{remark}
In order to understand the intuition behind \eqref{cen_dynamic} and \eqref{eq:rad}. Let $\hat q^{(j)}(x)$ be the probability that a point $x \in \mathbb{R}^n$ avoids a collision with the $j$-th obstacle:
\begin{align}
  \hat q^{(j)}(x,t)=1-\mathbb{P}\Big(\norm{x-O_p^{(j)}(t)}<\rad_o^{(j)}(t)\Big), \forall x\in \R^n.\label{eqn:q_hat}
\end{align}
To satisfy the avoidance specification, the computed tube $\Gamma(t)$ must lie
outside the $\varepsilon^{(j)}$ sub-level set of $\hat q(x,t) $ region shown in Figure~\ref{fig:intution} at a particular time $t\in \R^+_0$. Likewise, the tube center $\cen(t)$ must remain outside the light red region, corresponding to the $\varepsilon^{(j)}$ sub-level set of $q^{(j)}(x)$ at a particular time $t\in \R^+_0$.

The STT center \eqref{cen_dynamic} and radius \eqref{eq:rad} are designed to satisfy these requirements. Specifically, the probabilistic avoidance term in the center dynamics becomes active whenever $\cen(t)$ enters the outermost region shown in Figure~\ref{fig:intution}, i.e., when $\theta^{(j)}(t)\neq 0$ and prevents the center from entering the middle region. Meanwhile, the radius dynamics ensure that the tube itself remains outside the innermost region, as depicted in Figure~\ref{fig:intution}.
% \textcolor{red}{The intuition behind \eqref{cen_dynamic} is as follows: the first term pulls the center of the tube $\cen$ to reach the goal point $\eta$, ensuring the convergence to the goal with in finite time. The second term gets activated only when the probability of avoiding the collision $q^{(j)}$ is less than the user-defined defined probability threshold $p_d^{(j)}$, it pushes away the center $\cen$ of the tube away from the region where the chance of collision of the tube with the obstacle is greater than $1-\varepsilon^{(j)}$. The radius dynamics $\rad$ in \eqref{eq:rad} complement the behavior of the tube center by smoothly adjusting the radius of the tube such the tube $\Gamma(t)=\mathcal{B}(\cen(t),\rad(t))$ never intersect with unsafe set with the probability higher than $\varepsilon^{(j)}$ for each unsafe set.}
\begin{figure*}
  \centering
  \includegraphics[width=0.7\linewidth]{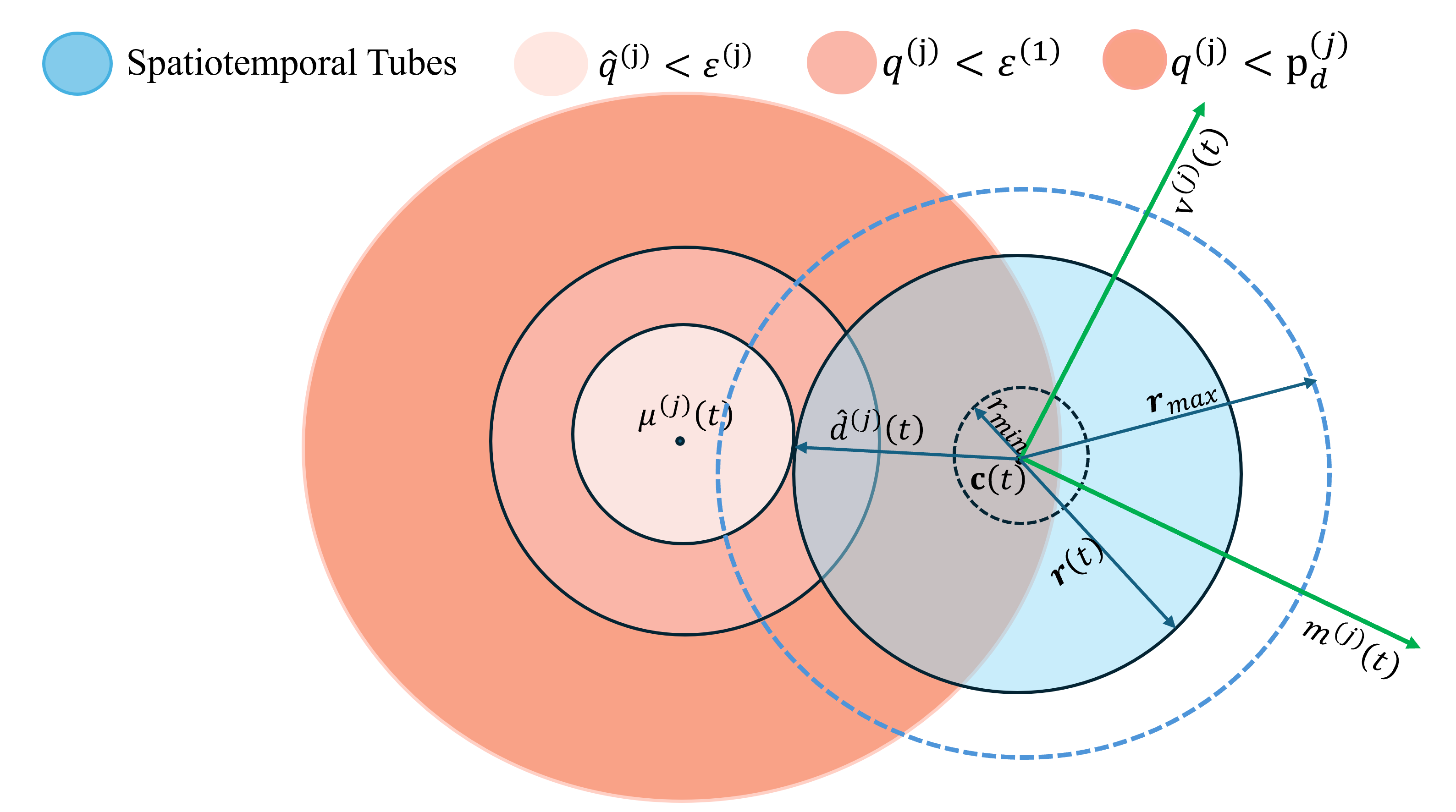}
  \caption{A schematic representation of the probabilistic avoidance mechanism corresponding to an uncertain obstacle $O_p^{(j)}(t)\sim\mathcal{N}(\mu^{(j)}(t),\Sigma^{(j)}(t))$. 
  % which includes the $\varepsilon^{(j)}$ sublevel set of the $\hat q^{(j)}(t)$, and the $\varepsilon^{(j)}$ and $p_d$ sublevel sets of $q^{(j)}(x,t)$ at a given time $t$. These regions correspond to an uncertain obstacle $O_p^{(j)}(t)\sim\mathcal{N}(\mu^{(j)}(t),\Sigma^{(j)}(t))$ and are depicted together with the spatiotemporal tube.
  }
  \label{fig:intution}
\end{figure*}

% {Given a time-varying center $\cen: \R_0^+ \rightarrow \R^n$ and a time-varying radius $\rad: \R_0^+ \rightarrow \R^n$, governed by the dynamics in Equations \eqref{cen_dynamic} and \eqref{eq:rad}, respectively, we define the STT $\Gamma(t) = \mathcal{B} (\cen(t), \rad(t))$ as a closed ball in $\R^n$ centered at $\cen(t)$ with radius $\rad(t)$:
% \begin{equation}\label{eqn:stt_ball}
%   \Gamma(t) = \mathcal{B} (\cen(t), \rad(t)) := \{ x\in \R^n \mid \|x-\cen(t)\| \leq \rad(t)\}, \quad \forall t \in \R_0^+ .
% \end{equation}
% }

The next theorem establishes the main result of the paper, showing that the designed STT satisfies the PrT-RAS specification.

% The next theorem guarantees that the designed STT adheres to the following three key conditions for satisfying PrT-RAS specifications. First, the tube starting from the given initial set reaches the target set within the finite time $t = t_c,t_c\in\R^+$. Second, the tube satisfies the avoid specification for each unsafe set with user defined probability for all times $t \in [0, \infty)$, ensuring safety. Finally, the radius of the tube remains strictly positive throughout the motion, guaranteeing that there is always a nonzero gap for the system trajectory to be safely enclosed within the tube.
% \begin{assumption}\label{ass_pmin}
%   At time $t=t_c$, the unsafe set is separated from the center of the tube $\cen(t_c)$ by at least a known minimum probability $p_{d}^{(j)}>0$, i.e.,
%   $$q^{(j)}(t)>p_d^{(j)} $$
% \end{assumption}

\begin{theorem}
  Let the time-varying center $\cen: \R_0^+ \rightarrow \R^n$ and $\rad: \R_0^+ \rightarrow \R^n$ evolve according to the dynamics in Equations \eqref{cen_dynamic} and \eqref{eq:rad}, respectively. Define the STT
  \begin{equation}\label{eqn:stt_ball}
    \Gamma(t) = \mathcal{B} (\cen(t), \rad(t)) := \{ x\in \R^n \mid \|x-\cen(t)\| \leq \rad(t)\}, \quad \forall t \in \R_0^+ .
  \end{equation}
  Then, the STT $\Gamma(t)$ satisfies all the conditions required to ensure the PrT-RAS specification:
  \begin{enumerate}
    \item [(i) ]The tube starts inside the initial set and reaches the target set in finite time $t_c\in \R_0^+$ and stays inside the target thereafter: 
    \begin{align}
    \Gamma(0)\subset\So,\Gamma(t)\subset \T,\forall t \in[t_c,\infty).
    \end{align}
    
    \item [(ii)] For each obstacle and at all times, the probability that the tube avoids the obstacle is lower bounded by the user-specified threshold $\varepsilon^{(j)}$:
    % The probability that the tube avoids the unsafe set is lower bounded by the user defined probability, or alternatively the probability of collision of the tube with the obstacle is upper bounded by $1-\varepsilon^{(j)},\forall j \in [1;n_o]$ for all time i.e.:
    \begin{align}
      \Prob\Big(\norm{\cen(t)-O_p^{(j)}(t)}\geq\rad_o^{(j)}(t)+\rad(t)\Big)>\varepsilon^{(j)},\forall j\in[1;n_o], t \in \R_0^+.
    \end{align}
    \item [(iii)] The tube radius always remains strictly positive: $\rad(t)\in \R ^+,\forall t\in\R_0^+$.
  \end{enumerate}
\end{theorem}
\begin{proof}
  We will prove each of the three claims in the theorem independently.\\
  (i) At $t=0$, the tube center starts from the point $s$, $\cen(0)=s$. 
  
  From Assumption \ref{ass_pmin}, for all $t\geq t_1$, the tube remains sufficiently far from all the unsafe regions, implying that $\theta^{(j)}=0,$ for all $j\in[1,n_o]$. Thus, for all $t \geq t_1$, Equation \eqref{cen_dynamic} simplifies to $\dot{\cen}=k_1(\eta-\cen(t))^{\frac{1}{3}}$. Solving this, we obtain
  \begin{align}
    \cen(t)=\eta-\big((\eta-\cen(t_1))^{2/3}-\frac{2}{3}k_1(t-t_1)\big)^{3/2},\label{eqn:cen_simplified}
  \end{align}
  where $c(t_1)$ is the location of the tube center at time $t=t_1$. 

Equating \eqref{eqn:cen_simplified} to the target point $\eta$ gives the convergence time $t_c$ as follows:
$t_c=t_1+\frac{3(\eta-c(t_1))^{\frac{2}{3}}}{2k_1} \in \R^+, \ \text{ at which } \ \cen(t_c) \rightarrow \eta.$
% \begin{align}
%   t_c=t_1+\frac{3(\eta-c(t_1))^{\frac{2}{3}}}{2k_1} \in \R^+, \ \text{ at which } \ \cen(t_c) \rightarrow \eta. \nonumber
% \end{align}
Thus, the tube center reaches the target set in finite time, establishing the finite-time convergence of $\cen(t)$. 

Also, we can write the solution for the dynamics of radius $\rad(t)$ as:
\begin{align}
  \rad(t)=-\frac{1}{\nu}ln(\e^{-\nu\rad_{max}}+e^{-\nu d(t)}), \label{eqn:rad_sol}
\end{align}
which represents a smooth approximation of the $\min$ operator. Accordingly, the tube radius satisfies the inequality:
\begin{align}
  \rad(t)\leq\min(\rad_{max},\min_{j=1,..,n_o}\hat d^{(j)}(t)). \label{eqn:rad:inequality}
\end{align}
From the above inequality, we can infer that $\rad(t)\leq\rad_{\max}\leq \min(d_S,d_T),\forall t \in \R_0^+$. Consequently, the STT satisfies $\Gamma(0)\subset \So$ and $\Gamma(t)\subset \T, \forall t \in [t_c,\infty)$, which ensures that the tube remains entirely within the initial and target sets during the respective phases of the task..

(ii) We now prove the second claim of the theorem in two parts. First, we show that $q^{(j)}(\,t)>\varepsilon^{(j)}$, for all $t\in\R_0^+$, meaning that the tube center avoids each unsafe set with a probability strictly greater than the required threshold $\varepsilon^{(j)}$ for all $j \in [1;n_o]$. In the second part, we show that $\rad(t)\leq \min_{j=[1; n_o]}\hat d ^{(j)}$, for all $t \in \R_0^+$, implying that $\rad(t)$ remains small enough such that the tube avoids the obstacle with a minimum probability of $\varepsilon^{(j)}$ for all $j \in [1;n_o]$.

\textbf{Part 1:} For all $j\in [1,n_o]$, define a time varying function $J^{(j)}(t)=q^{(j)}(\cen(t),t)-\varepsilon^{(j)}$ with time derivative $\dot J^{(j)}(t)=\dot q^{(j)}(\cen(t),t).$
Substituting the expression for $q^{(j)}(\cen(t),t)$ from \eqref{eqn:q_F} and the corresponding derivative $\dot q^{(j)}(\cen(t),t)$, into $\dot J^{(j)}(t)$, 
\begin{align}
  \dot J^{(j)}(t)&=-\frac{\partial}{\partial \beta^{(j)}}F_{\hat Z^{(j)}(t)}\big(\beta^{(j)}(t);n,\lambda^{(j)}(\cen{(t)},t))\dot \beta^{(j)}(t)-\frac{\partial}{\partial \lambda^{(j)}}F_{\hat Z^{(j)}(t)}\big(\beta^{(j)}(t);n,\lambda^{(j)}(\cen{(t)},t))\dot \lambda^{(j)}(\cen{(t)},t)\notag\\
  \dot \lambda^{(j)}(\cen{(t)},t)&=2(\cen(t)-\mu^{(j)})^\top(\dot \cen(t)-\dot\mu^{(j)}(t))-2\frac{\norm{\cen(t)-\mu^{(j)}(t)}^2}{\sigma^{(j)}(t)}\dot \sigma^{(j)}(t),\notag
\end{align}
where $\beta^{(j)}(t)=\big(\frac{\rad_s^{(j)}(t)}{\sigma^{(j)}(t)}\big)^2$. It can be observed that 
$\mathcal{P}_{\hat{Z}^{(j)}}(x;n,\lambda^{(j)})=$\\
$\frac{\partial}{\partial \beta^{(j)}}F_{\hat Z^{(j)}(t)}\big(\beta^{(j)}(t);n,\lambda^{(j)}(\cen{(t)},t))$ is the PDF of the random variable $\hat Z^{(j)}(t)$. For better notational clarity, we will represent it as $\mathcal{P}_{\hat{Z}^{(j)}}$ and $F_{\hat Z^{(j)}(t)}$.

% Next we substitute the expression of $\dot \lambda^{(j)}(\cen{(t)},t)$ in $\dot J^{(j)}(t)$:
% \begin{align}
%   \dot J^{(j)}(t)&=-\mathcal{P}_{\hat{Z}^{(j)}}\dot \beta^{(j)}(t)-2\frac{\partial}{\partial \lambda^{(j)}}F_{\hat Z^{(j)}(t)}(\cen(t)-\mu^{(j)})^\top\dot \cen(t)\notag\\&+2\frac{\partial}{\partial \lambda^{(j)}}F_{\hat Z^{(j)}(t)}(\cen(t)-\mu^{(j)})^\top\dot\mu^{(j)}(t)-2\frac{\norm{\cen(t)-\mu^{(j)}(t)}^2}{\sigma^{(j)}(t)}\dot \sigma^{(j)}(t),\notag
% \end{align}
Next, we substitute the expression of $\dot \cen(t)$ from \eqref{cen_dynamic} in the above equation:
\begin{align}
\dot J^{(j)}(t)&=-\mathcal{P}_{\hat{Z}^{(j)}}\dot \beta^{(j)}(t)-2k_1\frac{\partial}{\partial \lambda^{(j)}}F_{\hat Z^{(j)}(t)}(\cen(t)-\mu^{(j)})^\top (\eta-\cen(t))^{\frac{1}{3}}\notag\\
-2k_{2,j}&\frac{\partial}{\partial \lambda^{(j)}}F_{\hat Z^{(j)}(t)}\frac{\norm{\cen(t)-\mu^{(j)}(t)}^2}{q^{(j)}(\cen(t),t)-\varepsilon^{(j)}}-2k_{3,j}\frac{\partial}{\partial \lambda^{(j)}}F_{\hat Z^{(j)}(t)}(\cen(t)-\mu^{(j)})^\top v^{(j)}\notag\\
&+2\frac{\partial}{\partial \lambda^{(j)}}F_{\hat Z^{(j)}(t)}(\cen(t)-\mu^{(j)})^\top\dot\mu^{(j)}(t)-2\frac{\norm{\cen(t)-\mu^{(j)}(t)}^2}{\sigma^{(j)}(t)}\dot \sigma^{(j)}(t),\notag
\end{align}
Now we analyze the behavior of $\dot J^{(j)}(t)$ as $q^{(j)}(\cen(t),t)\rightarrow \varepsilon^{(j)}$.

From \eqref{prop2}, all terms in the expression for $\dot J^{(j)}(t)$ remain finite except the third term. As $q^{(j)}(\cen(t),t)$ approaches $\varepsilon^{(j)}$, the denominator of this third term tends to zero, causing the term to dominate. Furthermore, \eqref{prop2} gives $\frac{\partial}{\partial \lambda^{(j)}}F_{\hat Z^{(j)}(t)}<0$ and by definition $k_{2,j}>0$, which together imply 
$\dot J ^{(j)}(t)>0, \ \text{ whenever } \ q^{(j)}(\cen(t),t)\rightarrow \varepsilon^{(j)}$
Since, $J^{(j)}(0)>0$, we conclude
$J^{(j)}(t)>0 \implies q^{(j)}(\cen(t),t)>\varepsilon^{(j)}, \ \forall t \in \R_0^+.$ 
Hence, the STT center $\cen^{(t)}$ maintains a minimum separation from the region where the probability of collision with the uncertain obstacle is greater than $1-\varepsilon^{(j)}$. Repeating this argument for all $j \in [1 ;n_o]$, we can conclude that the STT center maintains the required probabilistic separation from each uncertain unsafe set for all time.

\textbf{Part 2:} In this part, we will show that 
\begin{align}
\rad(t)\leq \min_{j=[1; n_o]}\hat d ^{(j)},\forall t \in \R_0^+.\label{eqn:rad_cond}
\end{align}
We consider the following two cases for that:

\textbf{Case 1}: $\rad_{max}\leq \hat d ^{(j)},\forall j \in [1,n_o]$

From \eqref{eqn:rad:inequality}, it follows directly that
$\rad(t)\leq \rad_{max} \leq \min_{j=1,..,n_o}\hat d^{(j)}(t)).$

\textbf{Case 2:} $\rad_{max}> \hat d^{(j)}(t))$, for some $\hat j\in[1,n_o]$

Using the solution~\eqref{eqn:rad_sol} and the corresponding inequality~\eqref{eqn:rad:inequality}:
\begin{align}
  \rad(t)&=-\frac{1}{\nu}ln(\e^{-\nu\rad_{max}}+e^{-\nu d(t)})\leq\min(\rad_{max},\min_{j=1,..,n_o}\hat d^{(j)}(t))\leq\hat{d}^{(\hat j)}(t),
\end{align}
Therefore, in both the cases condition~\eqref{eqn:rad_cond} is satisfied, concluding the second part of our proof.

(iii) We had already established that $q^{(j)}(\cen(t),t)>\varepsilon^{(j)},\forall t \in\R_0^+$, so we use this condition in \eqref{eqn:m_crv} along with the monotonically increasing property of the CDF $F_{\hat Z^{(j)}}$ to get that $d(t)>\rad_{min},\forall t \in \R_0^+$. Substituting this into the solution of the $\rad(t)$ dynamics, we get:
\begin{align}
  \rad(t)>-\frac{1}{\nu}ln(\e^{-\nu\rad_{max}}+e^{-\nu \rad_{min}})>0, \quad \forall t \in \R_0^+.\nonumber
\end{align}
This implies that the tube radius remains strictly positive at all times.\\
Hence, the STT $\Gamma(t)$ reaches the target $\T$ in finite time $t_c$, avoids the unsafe set $\U(t)$ with a user-defined probability level, and maintains a guaranteed positive radius throughout. This completes the proof that the proposed STT satisfies the PrT-RAS specification.
\end{proof}
In order to satisfy the given PrT-RAS specification by any system \eqref{eqn:sysdyn}, the problem now boils down to designing a controller that constrains the output trajectory within the STT, i.e.,
\begin{align}\label{eqn:STT_const}
  y(t)\in \Gamma(t), \forall t \in \R_0^+.
\end{align}
\begin{lemma}\label{lemma_c}
  The tube center $\cen(t)$, radius $\rad(t)$ and their time derivatives $\dot{\cen}(t)$ and $\dot{\rad}(t)$ are all continuous and bounded for all $t \in \R_0^+$.
\end{lemma}

\begin{proof}
From the radius dynamics \eqref{eq:rad} and the smooth approximation of the $\min$ function, both $\rad(t)$ and $\dot{\rad}(t)$ are continuous and remain bounded for all $t \in \R_0^+$. Moreover, since $q^{(j)}(\cen(t),t)>\varepsilon^{(j)}$, for all $t \in \R_0^+$ and all $j\in [1,n_o]$, the center dynamics, which depend smoothly on the probability of avoiding the collision with unsafe sets and the target, ensure that $\cen(t)$ and $\dot{\cen}(t)$ are also continuous and bounded over the entire time horizon.
\end{proof}
\begin{remark}
% Modeling the uncertainty at the center $O_p^{(j)}$ of the obstacle $\mathcal{U}^{(j)},\forall j \in [1;n_o]$ with an isotropic covariance matrix (Definition \eqref{defn:obs}) allows us to express the collision-avoidance probability using the CDF of a non-central chi-squared distribution. If the uncertainties differ across dimensions, the formulation still remains valid, but the probability of avoiding the obstacle can no longer be computed in closed form; instead, it must be estimated through the sampling methods mentioned in \cite{monte_carlo}.
{Modeling the obstacle center uncertainty with an isotropic covariance (Definition \eqref{defn:obs}) allows the collision-avoidance probability to be expressed using the CDF of a non-central chi-squared distribution. If the uncertainty is anisotropic, the formulation remains valid, but the avoidance probability can no longer be expressed using known distribution and must instead be estimated using sampling-based methods such as those in \cite{monte_carlo}.}
\end{remark}

\section{Controller Design}
In this section, we will derive an approximation-free, closed-form control law that constrains the output of the system \eqref{eqn:sysdyn} within the STT $\Gamma(t)=\B(\cen(t),\rad(t))$ in \eqref{cen_dynamic},\eqref{eq:rad}, ensuring \eqref{eqn:STT_const} is satisfied. We leverage the lower-triangular structure of \eqref{eqn:sysdyn} and adopt an approach similar to backstepping, as proposed in \cite{STT_real_time}.

We start by designing an intermediate control input $r_2$ for the $x_1$ dynamics (output of the system) to enforce the STT constraint. Then we will follow the methodology discussed in \cite{das2024spatiotemporal}, then iteratively we will construct a sequence of intermediate control inputs $r_{k+1}$ which will ensure that each state $x_k$ tracks its reference signal $r_k$ for all $k\in[2;N]$. The final control input $u$ to the system will be $r_{N+1}$.

The steps to design control input are as follows:

\textbf{Stage $1$:} Given the STT $\Gamma(t)$ as defined in Equation \eqref{eqn:stt_ball}, let the normalized error $e_1(x_1,t)$ and the transformed error $\rho_1(x_1,t)$ be given by
\begin{align}
e_1(x_1,t) &= \frac{\norm{x_1(t) - \cen(t)}}{\rad(t)}, \nonumber \quad
\rho_1(x_1,t) = \ln\left( \frac{1 + e_1(x_1,t)}{1 - e_1(x_1,t)} \right).\nonumber
\end{align}
The intermediate control input $r_2(x_1,t)$ is given by: 
\begin{equation*}
  r_2(x_1,t) = -\kappa_1 \rho_1(x_1,t) \left( x_1(t)-\cen(t) \right), \kappa_1 \in \R^+.
\end{equation*}

\textbf{Stage $k$} ($k \in [2;N]$): In order to ensure that $x_k$ tracks the reference signal $r_k$ from Stage k-1, we define a time varying bound: $\gamma_{k,i}(t) = (p_{k,i} - q_{k,i})e^{-\mu_{k,i}t} + q_{k,i}$, and enforce: 
\begin{align*}
  -\gamma_{k,i}(t) \leq (x_{k,i}-r_{k,i}) \leq \gamma_{k,i}(t) \ \ \forall (t,i) \in \R_0^+ \times [1;n],
\end{align*}
where, $\mu_{k,i} \in \R_0^+$, and $p_{k,i}, q_{k,i} \in \R^+$ with $p_{k,i} > q_{k,i}$ are chosen such that the initial error is within the bounds: $|x_{k,i}(0) - r_{k,i}(0)| \leq p_{k,i}$.

Now, define the normalized error $e_k(x_{k},t)$, the transformed error $\rho_k(x_{k},t)$ and the diagonal matrix $\xi_k(x_{k},t)$ as
\begin{subequations} \label{eq:stage k}
  \begin{align}
   e_k(x_{k},t) &= [e_{k,1}(x_{k,1},t), \ldots, e_{k,n}(x_{k,n},t)]^\top \\
  &= (\textsf{diag}(\gamma_{k,1}(t),\ldots,\gamma_{k,n}(t)))^{-1} \left(x_{k} - r_k \right), \notag \\
  \rho_k(x_{k},t) &= \big[\ln\left(\frac{1+e_{k,1}(x_{k,1},t)}{1-e_{k,1}(x_{k,1},t)}\right), \ldots, \ln\left(\frac{1+e_{k,n}(x_{k,n},t)}{1-e_{k,n}(x_{k,n},t)}\right) \big]^\top, \\
  % \xi_k(x_{k},t) &= \frac{4 (\textsf{diag}(\gamma_{k,1}(t),\ldots,\gamma_{k,n}(t)))^{-1}}{1-e_k^\top(x_{k},t)e_k(x_{k},t)}.
    \xi_k(x_{k},t) &= 4 \big(\textsf{diag}(\gamma_{k,1}(t),\ldots,\gamma_{k,n}(t)) \big)^{-1}(I_n-\textsf{diag}(e_k \circ e_k))^{-1}.
\end{align}
\end{subequations}

The next intermediate control input $r_{k+1}(\bar x_k,t)$ is then:
\begin{equation*}
  r_{k+1}(\bar x_k,t) = - \kappa_k\xi_k(x_{k},t)\rho_k(x_{k},t), \kappa_k \in \R^+.
\end{equation*}
At the $N$-th stage, this intermediate input becomes the actual control input:
\begin{equation*}
  u(\bar x_N,t) = - \kappa_N\xi_N(x_{N},t)\rho_N(x_{N},t), \kappa_N \in \R^+.
\end{equation*}

We now state the main theorem, which guarantees that this controller enforces the desired PrT-RAS behavior.
\begin{theorem} \label{theorem_ras}
Consider the nonlinear MIMO system in \eqref{eqn:sysdyn} satisfying Assumptions \ref{assum:lip} and \ref{assum:pd}, a probabilistic temporal reach-avoid-stay (PrT-RAS) specification as defined in Definition \ref{def:prtras} in the presence of a probabilistic unsafe set defined in Definition \eqref{defn:obs} with pre-defined value of $\varepsilon^{(j)},\forall j \in [1;n_o]$, and the corresponding STT $\Gamma(t)$ as defined in Equation \eqref{eqn:stt_ball}.
  
  If the initial output is within the STT at time $t=0$: $y(0) \in \Gamma(0)$, then the closed-form control laws
  \begin{subequations}\label{eqn:Control_ras}
   \begin{align}
    r_2(x_1,t) &= -\kappa_1 \rho_1(x_1,t) \left( x_1(t)-\cen(t) \right), \kappa_1 \in \R^+, \\
    r_{k+1}(\bar x_k,t) &= - \kappa_k\xi_k(x_{k},t)\rho_k(x_{k},t), k \in [2;N-1], \\
    u(\bar x_N,t) &= - \kappa_N\xi_N(x_{N},t)\rho_N(x_{N},t),
  \end{align}  
  \end{subequations}
  ensure that the system output remains within the STT: 
  $y(t) = x_1(t) \in \Gamma(t), \forall t \in \R^+_0,$
  thereby satisfying the desired PrT-RAS specification.
\end{theorem}
% \begin{proof}
%   % Following the proof of Theorem 4.1 in \eqref{}, we can proof that the scheme of control input defined in \eqref{eqn:Control_ras} enforce the output trajectory inside the STT. 
% \end{proof}

\begin{proof}
The proof of the correctness of the control law  has been omitted due to lack of space and follows the same arguments as \cite[Theorem 4.1]{STT_real_time}.
\end{proof}

\section{Case Studies}
To validate the effectiveness of the proposed real-time navigation in an uncertain environment using the STT framework, we present three case studies: a 2D omnidirectional mobile robot, a 3D UAV, and a 7-DOF manipulator. The results include the hardware experiments for the mobile robot and manipulator case, which will demonstrate their real-world applicability.
\subsection{2D Omnidirectional Robot}

\textbf{Hardware Experiments:} To demonstrate the effectiveness of the proposed framework in a real-world scenario, we conducted hardware experiments with an omnidirectional robot operating in an uncertain dynamic environment. The task assigned to the robot is to start from the given initial set and reach the target set, while avoiding four uncertain obstacles with different levels of uncertainty $(\sigma^{(j)})^2$, and minimum probability levels $\varepsilon^{(j)},\forall j\in[1;4]$, the values of $((\sigma^{(j)})^2,\varepsilon^{(j)})$ used in each of the case is mentioned in the first time stamp of the Figures \ref{fig:combined_hw}. In all cases, the radius of each obstacle is kept constant at $\rad^{(j)}_o=0.15m$. The tube radii were bounded by $\rad_{min}=0.21$, $\rad_{max}=0.27m$, the center dynamics parameters in \eqref{cen_dynamic} are chosen as follows: $k_1=0.7$,$k_{2,j}=k_{3,j}=0.4,\forall j \in [1;4]$ and we choose $p_d^{(j)}=0.9999,\forall j \in [1;n_o]$.

To investigate how different uncertainty levels impact the system, we examine three cases. The obstacles share the same initial position, velocity, and radius in all three cases, and only $(\sigma^{(j)})^2$ and $\varepsilon^{(j)}$ are changed between cases and the values used for each cases are shown in the Figures \ref{fig:combined_hw}. We illustrate the uncertainty around each obstacle using three red-shaded regions corresponding to the $p_d^{(j)}$, $\varepsilon^{(j)}$, and $0.7$ sub-level sets of $\hat{q}(x,t)$. To ensure avoidance with a probability of at least $\varepsilon^{(j)}$, the tube must stay outside the middle region at all times. Although the initial position and velocity of each obstacle remain the same across all cases, the robot’s trajectory (shown in blue in Cases~$1-3$ in Figure~\ref{fig:combined_hw} differs significantly. This variation is primarily driven by the uncertainty level and the user-defined probability threshold $\varepsilon^{(j)}$. Obstacles with higher uncertainty or higher $(\sigma^{(j)})^2$ values produce larger probabilistic avoidance regions (larger sub-level sets of $\hat q(x,t)$), leading the robot to steer farther away from their mean positions. The full video of the hardware experiment is available at \href{https://youtu.be/-SaFbrn9BJM}{Link}.

 % \begin{figure}[t]
 %    \centering
 %    \includegraphics[width=0.8\linewidth]{figure_3.png}
 %    \caption{Case 1}
 %    \label{fig:2d_hw_case1}
 %  \end{figure}
 %  \begin{figure}[t]
 %    \centering
 %    \includegraphics[width=0.8\linewidth]{figure_4.png}
 %    \caption{Case 2}
 %    \label{fig:2d_hw_case2}
 %  \end{figure}
 %   \begin{figure}[t]]
 %    \centering
 %    \includegraphics[width=0.8\linewidth]{figure_5.png}
 %    \caption{Case 3}
 %    \label{fig:2d_hw_case3}
 %  \end{figure}
\begin{figure}[]
  \centering
  % Subfigure 1
  \begin{subfigure}{0.95\linewidth}
    \centering
    \includegraphics[width=\linewidth]{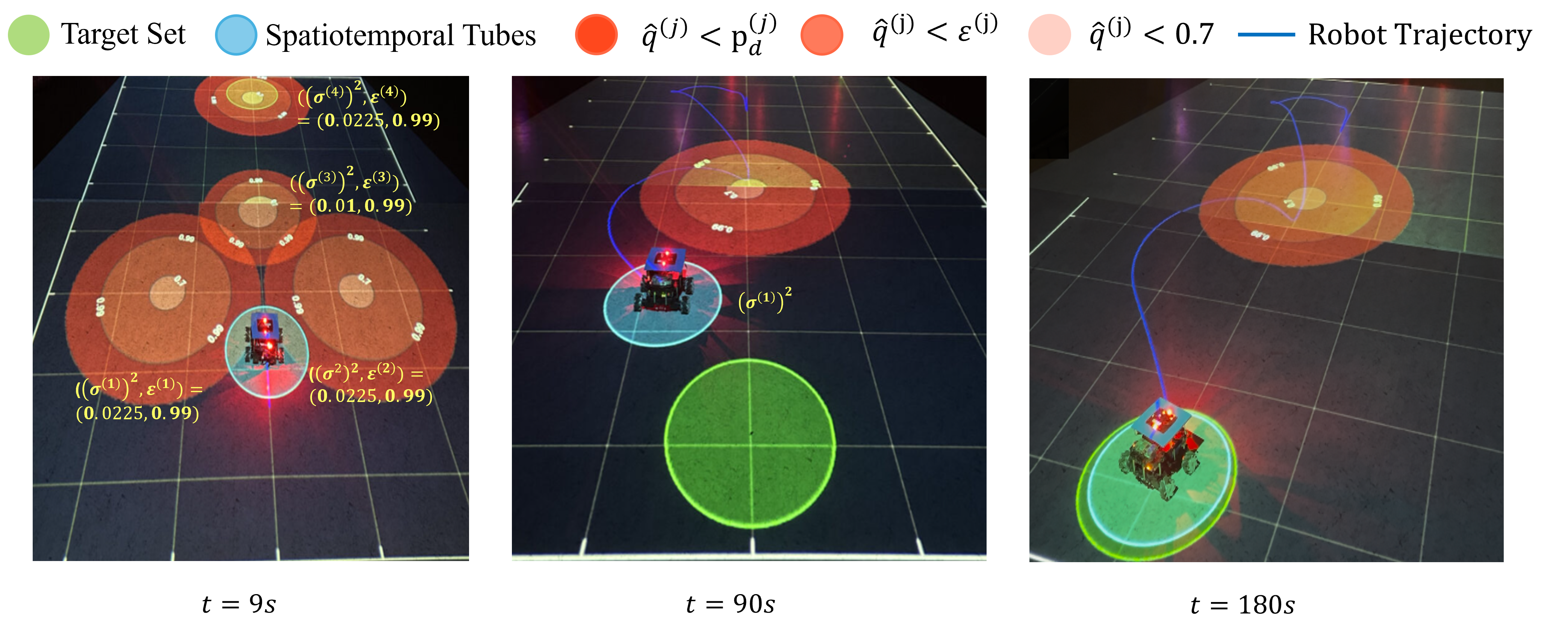}
    \caption{Case 1}
    \label{fig:2d_hw_case1}
  \end{subfigure}
  \centering
  \begin{subfigure}{0.95\linewidth}
    \centering
    \includegraphics[width=\linewidth]{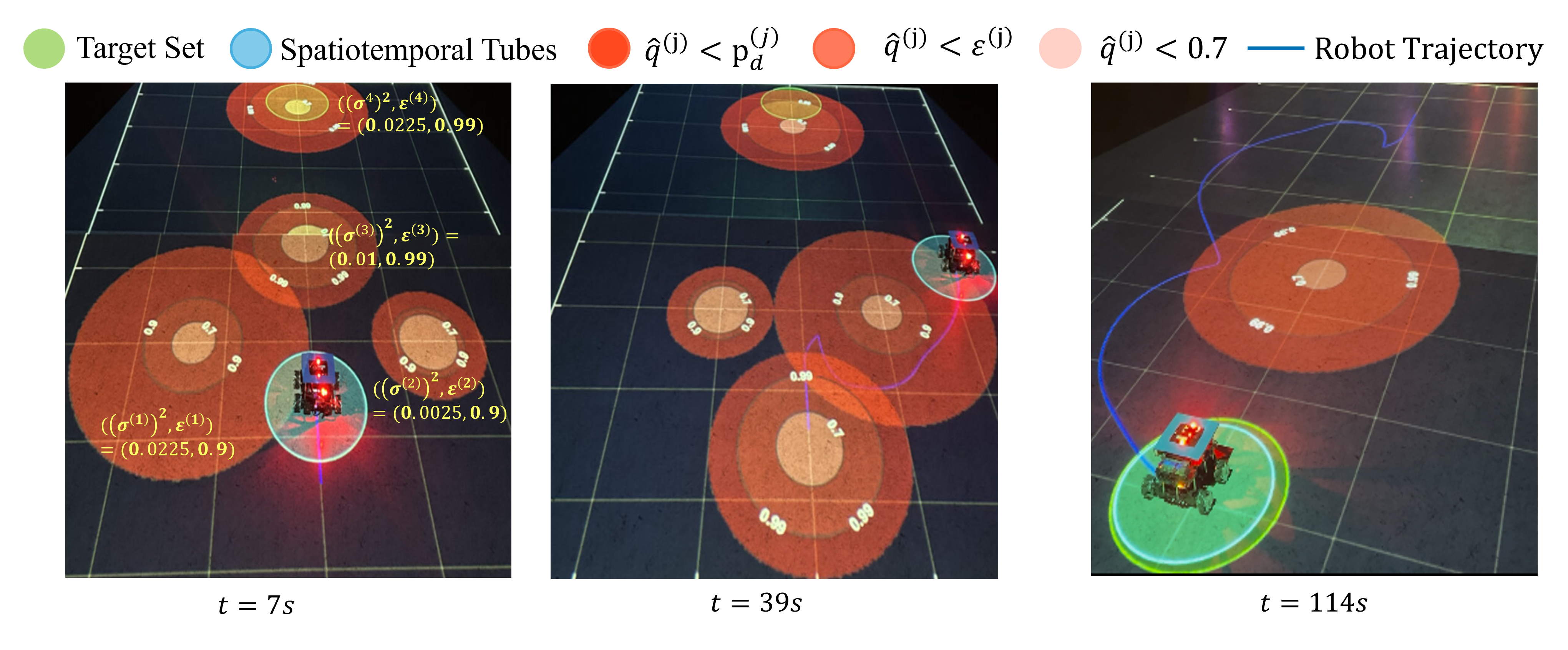}
    \caption{Case 2}
    \label{fig:2d_hw_case2}
  \end{subfigure}
  \begin{subfigure}{0.95\linewidth}
    \centering
    \includegraphics[width=\linewidth]{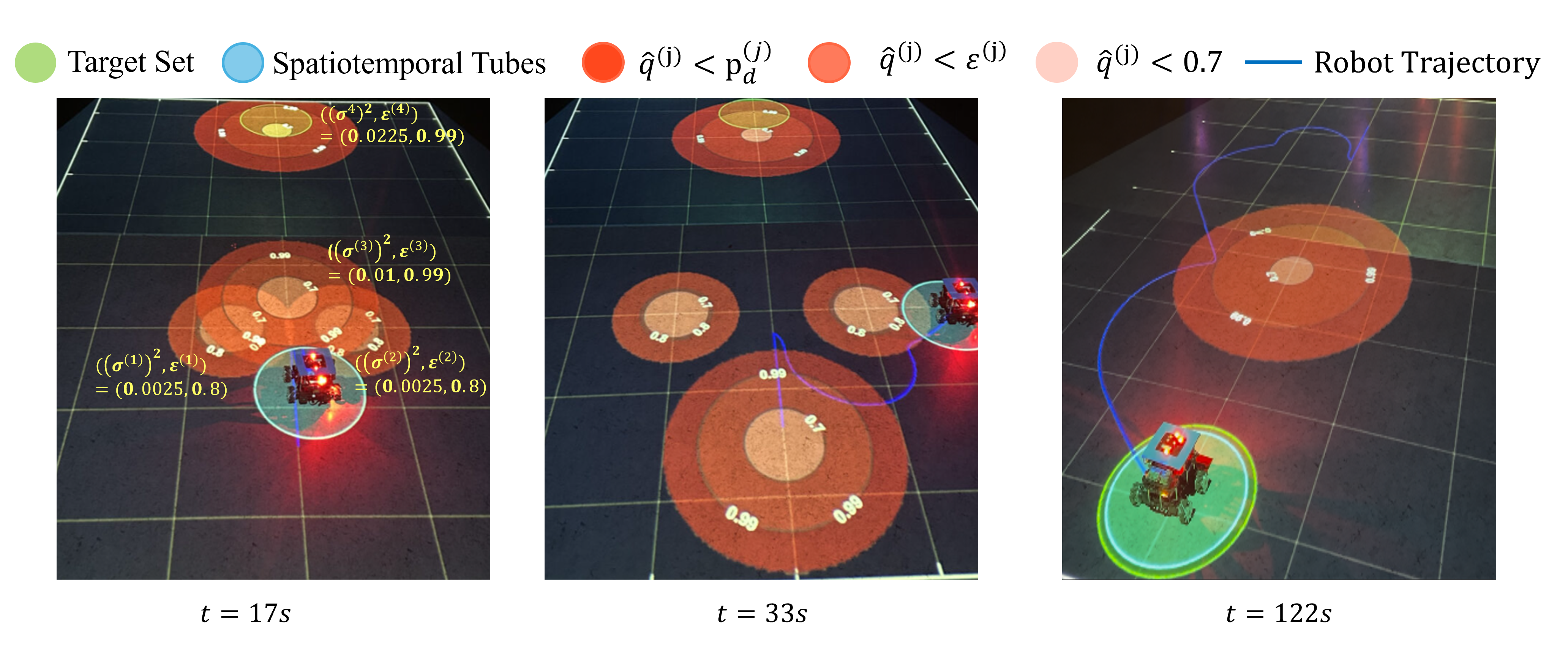}
    \caption{Case 3}
    \label{fig:2d_hw_case3}
  \end{subfigure}
  \caption{Hardware demonstration of a omnidirectional robot in three different dynamic uncertain environment, \href{https://youtu.be/-SaFbrn9BJM}{Video}}
  \label{fig:combined_hw}
\end{figure}

\textbf{Simulation Studies:} To demonstrate the effectiveness of the proposed approach in cluttered environments, we consider a scenario in which an omnidirectional robot, with dynamics adapted from \cite{NAHS}, navigates an environment containing $n_o = 50$ uncertain obstacles. The uncertainty level, mean position, and probability parameter $\varepsilon^{(j)}$ for each obstacle are selected at random, while their radii are kept constant to facilitate visualization of how probabilistic avoidance varies with uncertainty $(\sigma^{(j)})^2$. The robot’s task is identical to the hardware experiment, and it reaches the target set at $t_c = 55\,\text{s}$, as shown in Figure\ref{fig:2d_sim_1}. In this case the tube radii is bounded between $\rad_{\min} = 0.1$ and $\rad_{\max} = 0.8$. The center dynamics parameters in \eqref{cen_dynamic} are chosen as $k_1 = 0.35$ and $k_{2,j} =0.4, k_{3,j} = 0.8$ for all $j \in [1;4]$. In Figure~\ref{fig:2d_sim_1}, the uncertain unsafe regions are represented using the $\varepsilon^{(j)}$ and $0.7$ sub-level sets of $\hat{q}^{(j)}(x,t)$. The robot's trajectory (black solid line) in three different time instances, along with the tube synthesis with time, is shown in Figure \ref{fig:2d_sim_1}, it is clearly evident that the robot takes a longer detour when avoiding a more uncertain obstacle. The full simulation video is available at \href{https://youtu.be/-SaFbrn9BJM}{link}.
\begin{figure}[]
  \centering
  \includegraphics[width=1\linewidth]{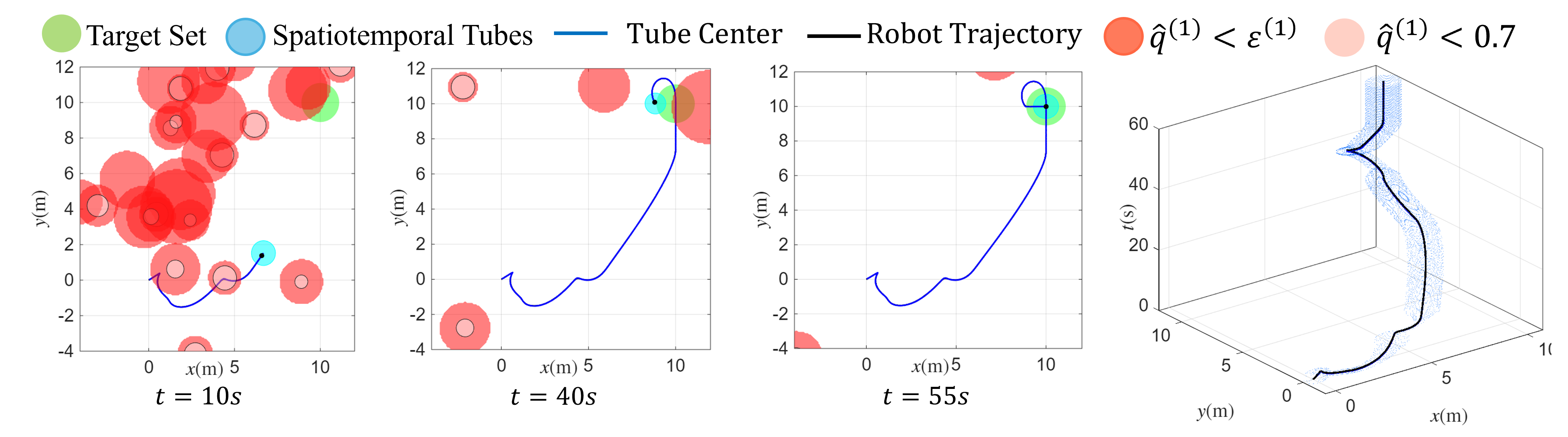}
  \caption{Simulation of an omnidirectional mobile robot in a 2D uncertain environment, along with a 3D plot showing tube synthesis with time $t$,  \href{https://youtu.be/-SaFbrn9BJM}{Video}}
  \label{fig:2d_sim_1}
\end{figure}
% In addition to illustrating probabilistic avoidance, this example also highlights an advantage of fixed-time convergence when operating in environments with dynamic obstacles. As shown in the \href{https://youtu.be/-SaFbrn9BJM}{Video}, the robot initially reaches the target set at approximately $t = 27\,\text{s}$, but subsequently exits the set to avoid a newly approaching uncertain obstacle. It then re-enters and remains within the target set at $t = 55\,\text{s}$. Handling such dynamics and real-time cases would be more challenging under prescribed-time convergence constraints.
{Results shown in Figure \ref{fig:2d_sim_1} not only illustrates probabilistic avoidance but also shows why fixed-time convergence is useful in environments with dynamic obstacles. As seen in the \href{https://youtu.be/-SaFbrn9BJM}{Video}, the robot reaches the target set at about $27\,\text{s}$, leaves it to avoid a newly approaching uncertain obstacle, and returns to the set by $55\,\text{s}$. Such behavior would be difficult to accommodate under a prescribed-time convergence requirement.}

\subsection{UAV}
We consider a UAV navigating in an uncertain 3D environment, where the UAV follows the second-order dynamics adapted from \cite{STT_real_time}. The tube radius is bounded between $\rad_{\min} = 0.3$ and $\rad_{\max} = 0.9$. The center dynamics parameters in \eqref{cen_dynamic} are chosen as $k_1 = 0.31$ and $k_{2,j} = k_{3,j} = 0.03$ for all $j \in [1;4]$.
\begin{figure}[t]
  \centering
  \includegraphics[width=0.86\linewidth]{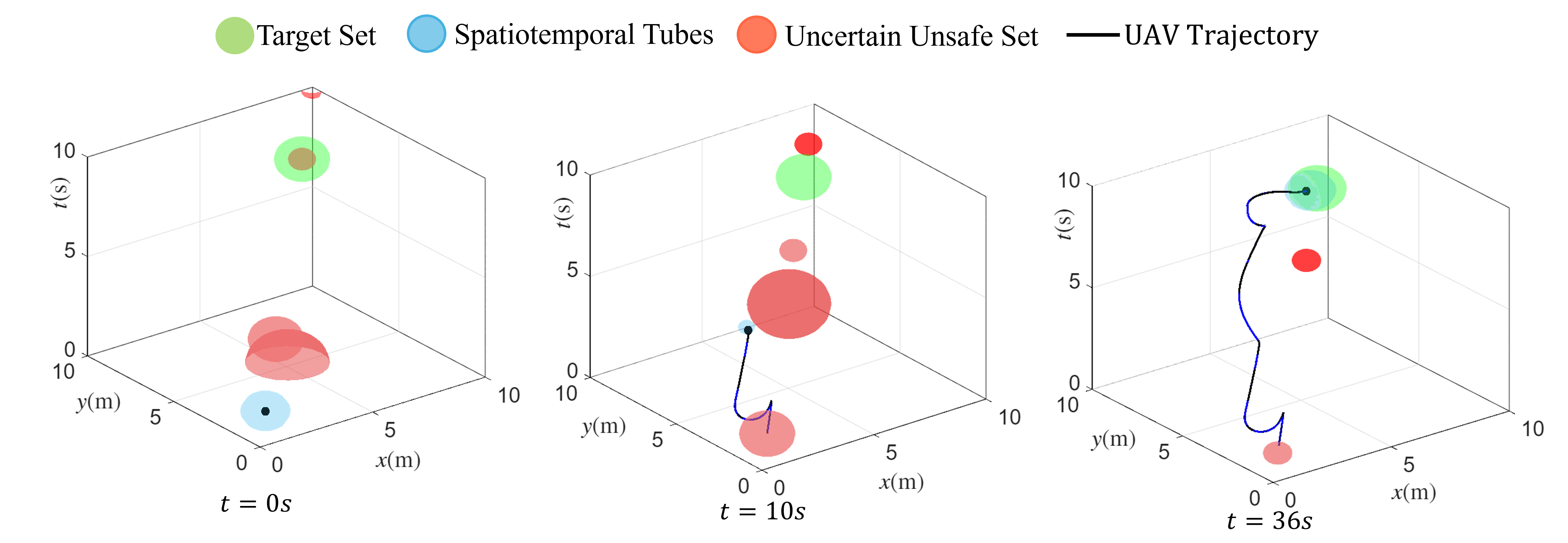}
  \caption{UAV simulation in a 3-D uncertain environment, where obstacles with higher uncertainty or alternatively higher $(\sigma^{(j)})^2$ values are depicted as dark red spheres, and those with lower uncertainty appear in lighter shades,  \href{https://youtu.be/-SaFbrn9BJM}{Video}}
  \label{fig:sim_3d}
\end{figure}
In Figure~\ref{fig:sim_3d}, obstacles with higher uncertainty are depicted in dark red, while those with lower uncertainty appear in lighter red. Unlike the 2D cases, we do not visualize the probabilistic unsafe regions directly; instead, each obstacle is represented as a ball centered at its mean with an arbitrarily chosen radius $\rad_o$, and their velocity, radius, uncertainty in position are selected randomly.

Figure~\ref{fig:sim_3d} illustrates the UAV’s trajectory in black solid lines along with the STT  at three time instances. Similar to the 2D scenario, the UAV takes a wider detour and behaves more cautiously around obstacles with higher uncertainty and larger $\varepsilon^{(j)}$, while it approaches closer to obstacles with lower uncertainty and smaller $\varepsilon^{(j)}$. A full simulation video is available at \href{https://youtu.be/-SaFbrn9BJM}{link}.

\subsection{7-DOF Manipulator}
The robotic manipulator used in our case study is a 7-DOF Franka Research 3 arm, commonly used for human–robot collaboration. We conduct two main experiments. In the first, the robot performs a pick-and-place task with two uncertain obstacles, repeating the trial with swapped uncertainty levels to observe how varying obstacle uncertainty affects the trajectory. In the second experiment, we use the same obstacle setup but introduce external jerks to evaluate the controller’s robustness to disturbances.

\textbf{Without Disturbance:} In this experiment the robot is assigned a pick-and-place task in the presence of uncertain obstacles as shown in Figure~\ref{fig:man_combined}. In Case 1 (Figure \ref{fig:man_2}), the blue obstacle has a higher uncertainty compared to the yellow one. As a result, the manipulator’s end effector takes a wider detour around the blue obstacle. When the uncertainty values are swapped, as illustrated in Case 2 (Figure \ref{fig:man_3}), the manipulator passes closer to the second obstacle and takes a larger detour around the other one. The complete video is available at \href{https://youtu.be/-SaFbrn9BJM}{link}

\textbf{With Disturbance:} In order to test the framework in the presence of external disturbance in the system, we consider the same task as in the case of without disturbance in the presence of the uncertain obstacle and apply sudden jerks to demonstrate disturbance rejection. In all cases, the controller achieved the required specification, as shown in the video available at \href{https://youtu.be/-SaFbrn9BJM}{link}
% \begin{figure*}
%   \centering
%   \includegraphics[width=1.0\linewidth]{figure_7.png}
%   \caption{}
%   \label{fig:man_1}
% \end{figure*}
\begin{figure}[t]
  \centering
  % Subfigure 1
  \begin{subfigure}{1\linewidth}
    \centering
    \includegraphics[width=0.8\linewidth]{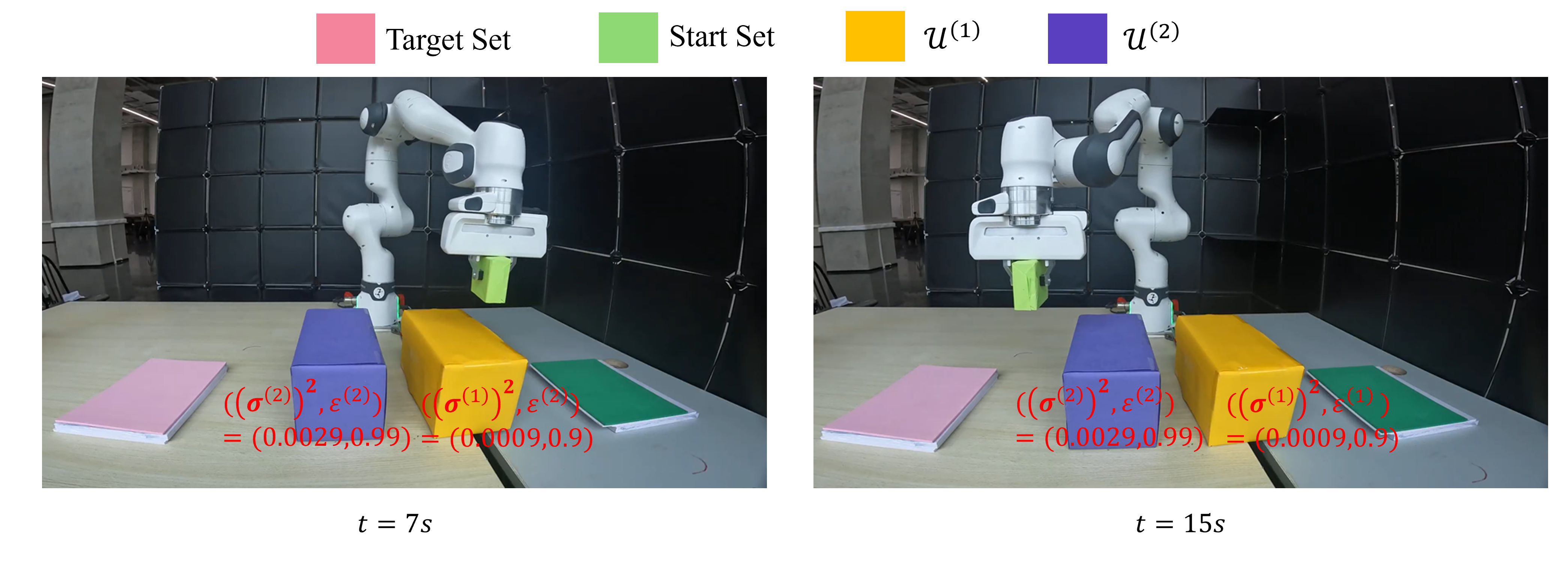}
    \caption{Case 1: uncertainty or $(\sigma^{(j)})^2$ associated with the first obstacle is lower than the second obstacle}
    \label{fig:man_2}
  \end{subfigure}

  % Subfigure 2
  \begin{subfigure}{1\linewidth}
    \centering
    \includegraphics[width=0.8\linewidth]{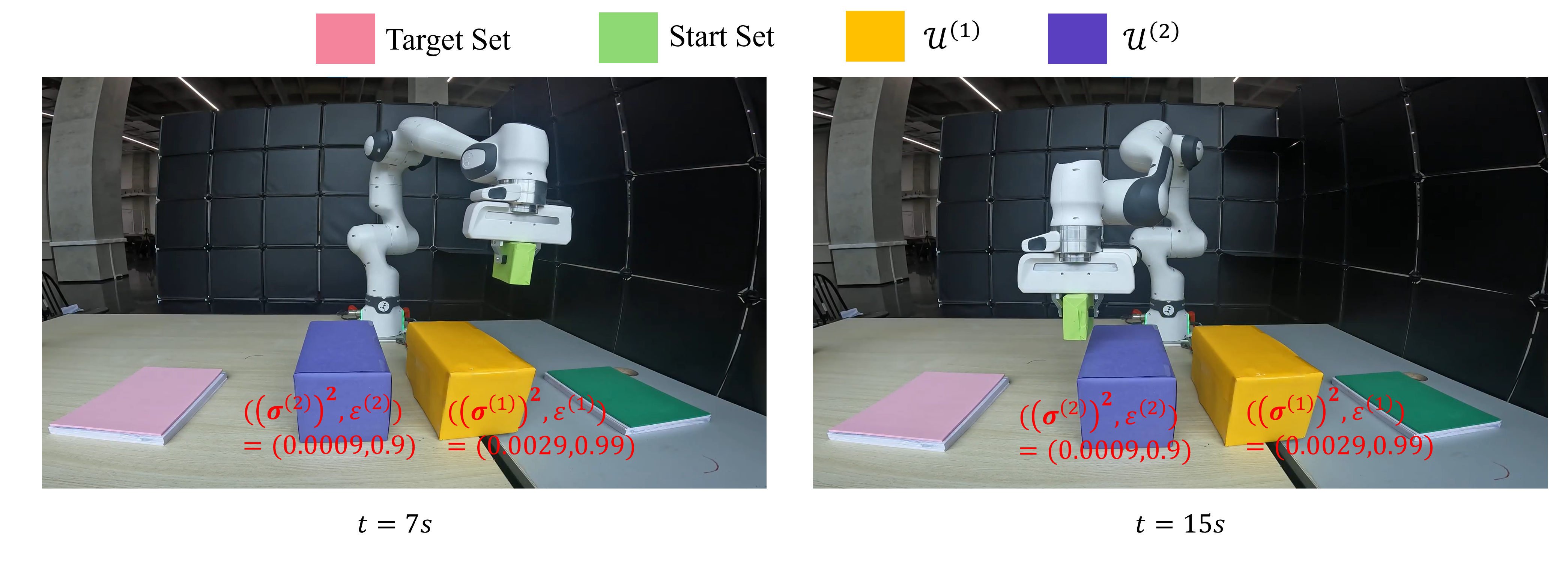}
    \caption{Case 2: uncertainty or $(\sigma^{(j)})^2$ associated with the first obstacle is higher than the second obstacle}
    \label{fig:man_3}
  \end{subfigure}

  \caption{Snapshot of hardware experiment at two different time stamps of a 7-DOF manipulator assigned with the task of pick and place in two different uncertain environments, \href{https://youtu.be/-SaFbrn9BJM}{Video}}
  \label{fig:man_combined}
\end{figure}
\section{Discussion}
Despite increases in system dimensionality from a 2-D mobile robot to a 14-D (7-DoF) manipulator and substantial growth in environmental complexity, ranging from 4 to 50 obstacles, the framework consistently achieves real-time performance. Hardware experiments naturally introduce unknown disturbances and sensor noise, with the manipulator additionally experiencing abrupt external jerks, thereby demonstrating the robustness of the proposed method to system-level uncertainties. Furthermore, across all scenarios and varying levels of perception uncertainty, the approach reliably maintains probabilistic safety, underscoring the formal, provable probabilistic guarantees afforded by the STT framework.

\section{Conclusion}
 In this work, we presented a real-time framework for synthesizing Spatiotemporal Tubes (STTs) that ensure safe control of nonlinear pure-feedback systems with unknown dynamics under PrT-RAS tasks in the presence of uncertain unsafe sets. The proposed method continuously adapts the tube based on real-time perceptual information, adjusting its center and radius to satisfy avoidance requirements with a user-specified minimum probability. We establish formal guarantees of probabilistic safety and finite-time convergence, and we demonstrate the effectiveness and scalability of the approach through simulations on a 2D mobile robot and a 3D UAV in cluttered uncertain environments, as well as hardware experiments on a 2D mobile robot and a 7-DOF manipulator.

In this work, we assumed Gaussian uncertainty for the unsafe sets; future research will explore a distributionally robust formulation. We also plan to extend the framework to incorporate high-level temporal logic specifications such as Probabilistic LTL (pLTL) and Probabilistic Signal Temporal Logic (PrSTL).

\bibliographystyle{unsrt} % We choose the "plain" reference style
\bibliography{sources} % Entries are in the refs.bib file

\appendix
\section{Proof of Proposition \eqref{prop:1}}\label{appendix_1}
\begin{proof}
  From Definition \eqref{defn:obs} we have $O_p^{(j)}(t) \sim \mathcal{N}\Big(\mu^{(j)}(t), \Sigma^{(j)}(t)\Big)$. Using the linearity of Gaussian distributions, we introduce the normalized random vector defined as $Z^{(j)}(t)=[Z_1^{(j)}(t),\ldots,Z^{(j)}_n(t)]^\top$ 
\begin{align}
  Z^{(j)}(t)=\frac{\cen(t)-O_p^{(j)}(t)}{\sigma^{(j)}(t)} \sim \mathcal{N}\Big(\frac{\cen(t)-\mu^{(j)}(t)}{\sigma^{(j)}(t)}, \mathbf{I}_n\Big),
\end{align}
where $Z^{(j)}(t)$ is a standard normal Gaussian vector.
Next, we define a random variable:
\begin{align}
    \hat Z^{(j)}(t)=\norm{Z^{(j)}(t)}^2=\sum_{i=1}^{n}({Z_i^{(j)}(t)})^{2},\label{eqn:z_norm}
\end{align}
% \begin{align}
%   \hat Z^{(j)}(t)&=\norm{Z^{(j)}(t)}^2=\sum_{i=1}^{n}({Z_i^{(j)}(t)})^{2},\label{eqn:norm_z}
% \end{align}
 which is the sum of squares of independent Gaussian random variables with nonzero means and unit variance. Therefore, by \cite{chi_compute}, $\hat Z^{(j)} $ follows a non-central chi-square distribution with degrees of freedom equal to the dimension of the state space and non centrality parameter $\lambda^{(j)}(\cen{(t)},t)=\frac{\norm{\cen(t)-\mu^{(j)}(t)}^2}{(\sigma^{(j)}(t))^2}$. 
By applying basic algebraic manipulation to \eqref{eqn:prob_q}, we obtain
\begin{align}
    q^{(j)}(\cen(t),t)=1-\mathbb{P} \Bigg( \frac{\norm{\cen(t)-O_p^{(j)}(t)}^2}{(\sigma^{(j)}(t))^2}<\frac{(\rad_s^{(j)})^2(t)}{(\sigma^{(j)}(t))^2}\Bigg)\notag
\end{align}
 % \begin{align}
 %   &q^{(j)}(\cen(t),t)=1-\mathbb{P} \Bigg( \frac{\norm{\cen(t)-O_p^{(j)}(t)}^2}{(\sigma^{(j)}(t))^2}<\frac{(\rad_s^{(j)})^2(t)}{(\sigma^{(j)}(t))^2}\Bigg). \nonumber
 % \end{align}
then using definition of $\hat Z^{(j)}(t)$ in \eqref{eqn:z_norm} we can rewrite the expression of $q^{(j)}(\cen(t),t)$ in terms of the CDF $F_{\hat Z^{(j)}}$ as:
 \begin{align}
     &q^{(j)}(\cen(t),t)=1-F_{\hat Z^{(j)}(t)}\Big(\big(\frac{\rad_s^{(j)}(t)}{\sigma^{(j)}(t)}\big)^2;n,\lambda^{(j)}(\cen{(t)},t)\Big). \notag
 \end{align}
\end{proof}
\section{Non-central chi squared distribution}\label{prop2}
\begin{proposition}
  The Cumulative Distribution function (CDF) and Probability Distribution Function (PDF) of a non-central chi-squared distribution follow the following property.
  \begin{enumerate}
  \item The PDF of the noncentral chi-square random variable, 
  $\mathcal{P}_{\hat{Z}^{(j)}}(x; n, \lambda^{(j)})$, is positive and finite for all 
  admissible parameter values. Precisely, for any $x \in \mathbb{R}^+$, $n \in \mathbb{N}$, and 
  $\lambda^{(j)} \in \mathbb{R}^+$, one has $\mathcal{P}_{\hat{Z}^{(j)}}(x; n, \lambda^{(j)}) > 0$, and it remains finite.
 \item The partial derivative of the cumulative distribution function (CDF) with respect to the 
  non-centrality parameter $\lambda^{(j)}$ is finite and strictly negative i.e. $\frac{\partial}{\partial \lambda^{(j)}}
    F_{\hat{Z}^{(j)}}(x; n, \lambda^{(j)})
    = -\mathcal{P}_{\hat{Z}^{(j)}}(x; n, \lambda^{(j)}) < 0.$   
  % \item The partial derivative of the cumulative distribution function (CDF) with respect to the 
  % non-centrality parameter $\lambda^{(j)}$ is finite and strictly negative. In particular, $\frac{\partial}{\partial \lambda^{(j)}}
  %   F_{\hat{Z}^{(j)}}(x; n, \lambda^{(j)})
  %   = -\mathcal{P}_{\hat{Z}^{(j)}}(x; n, \lambda^{(j)}) < 0,$   
  % which confirms that the CDF is monotonically decreasing as $\lambda^{(j)}$ increases.
\end{enumerate}

  \begin{proof}
      The detailed proof can be found in \cite[Chapter 29]{non_centr_proofs}.
    \end{proof}
\end{proposition}

\end{document}